\documentclass[10pt]{article}
\usepackage[preprint]{tmlr}
\usepackage{amsmath,amssymb,amsthm}
\usepackage{graphicx}
\usepackage{booktabs}
\usepackage{array}
\usepackage{microtype}
\usepackage{xcolor}
\definecolor{linkc}{rgb}{0,0.2,0.55}
\usepackage[colorlinks=true,linkcolor=linkc,citecolor=linkc,urlcolor=linkc]{hyperref}
\usepackage{url}
\newcommand{\doi}[1]{doi: \href{https://doi.org/#1}{\nolinkurl{#1}}}
\usepackage[most]{tcolorbox}
\usepackage{enumitem}

\newcommand{\btheta}{\boldsymbol{\theta}}
\newcommand{\R}{\mathbb{R}}
\newcommand{\panel}[1]{\textbf{#1}}
\newtheorem{theorem}{Theorem}
\newcommand{\restatelabel}{}
\newtheorem*{restatement}{Theorem~\ref{\restatelabel}}

\newcounter{restateeq}

\title{Double descent is the principle of least action}
\author{\name Congzhou M Sha \email consha@sas.upenn.edu \\
      \addr Penn Medicine Doylestown Hospital, 595 W State St, Doylestown, PA 18901, United States}

\def\month{MM}  
\def\year{YYYY} 
\def\openreview{\url{https://openreview.net/forum?id=XXXX}} 

\begin{document}
\maketitle
\begin{abstract}
\noindent The test error of a model plotted against its number of parameters $d$ falls, peaks when the model can just fit the training data, and falls again, exhibiting the double descent phenomenon. We explain the phenomenon with statistical mechanics. The training trajectory of a stochastic gradient-based method is a particle wandering over the energy landscape of the training loss at an induced temperature $T$, and a run that has equilibrated visits every parameter vector of a given training loss equally often, the fundamental postulate of statistical mechanics, with probability given by the Boltzmann distribution. Because training starts at an initial point and has only finite time to diffuse, it carries an effective weight decay, which makes every parameter a quadratic degree of freedom. The equipartition theorem then distributes the energy among the $d$ degrees of freedom in shares of $T/2$, so at a fixed training loss adding parameters lowers the temperature and drives the Boltzmann distribution toward the stationary path. Finally, adding parameters can only lower the $L^2$ norm of the stationary path, so a solution sampled at fixed loss is less likely to be large with increasing $d$, effectively increasing weight regularization.
\end{abstract}

\section{Introduction}
Consider the simple regression task as described by \citet{nakkiran2019blog,nakkiran2020deep}: fitting 20 noisy samples of a cubic function on the interval $[-1,1]$ using polynomials with varying numbers of parameters $d$ (Fig.~\ref{fig:main}\panel{A--D}). $d$ counts all the coefficients of the polynomial including the constant term, so a polynomial of degree $d-1$ has $d$ parameters. $d=2$, i.e. linear regression $f(x)=\theta_1x+\theta_0$, produces an underfit whereas $d=4$, i.e. the cubic $f(x)=\theta_3x^3+\theta_2x^2+\theta_1x+\theta_0$, is the intended ground truth solution. For $d=20$, there are exactly as many coefficients as there are points, so exactly one polynomial passes through all of the training data, and it oscillates wildly between the data. Surprisingly, fitting a $d=1000$ curve also passes through every data point, yet it tracks the true cubic nearly as well as the $d=4$ solution within the training data range, but without the enormous oscillations of the $d=20$ polynomial. \par Plotting the test error as a function of $d$ (Fig.~\ref{fig:main}\panel{E}) shows a classical U-shape initially, with a peak at $d=n$, however there is a long decreasing tail to the right. This double descent phenomenon~\citep{belkin2019reconciling,nakkiran2020deep} appears in deep networks as a function of layer width and number of training epochs. Exact treatments exist for linear and random-feature models~\citep{hastie2022surprises,mei2022generalization,bartlett2020benign}, and the threshold has been identified with a jamming transition~\citep{spigler2019jamming,geiger2020scaling} and with broken ergodicity~\citep{li2026broken}. The aim of this work is to provide further elementary physical intuition for the double descent phenomenon, with a pedagogical focus so that it is understandable to undergraduates.

\begin{figure}[!t]
  \centering
  \includegraphics[width=\textwidth]{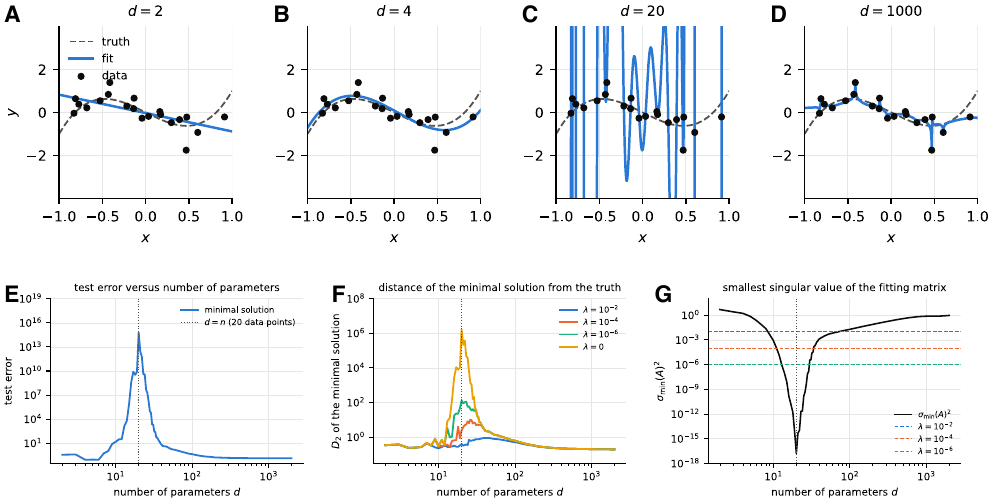}
  \caption{\panel{A--D} Legendre-basis fits with $d$ parameters (degree $d-1$, constant term included) to $n=20$ points $y_i=3x_i^3-2x_i+\xi_i$, $\xi_i\sim\mathcal N(0,0.4^2)$: least squares for $d\le n$, minimum-norm interpolation for $d>n$. \panel{E} Test error of the minimal solution, the least squares fit for $d\le n$ and the minimum-norm interpolant for $d>n$, which peaks at $d=n$. \panel{F} Root mean square distance from the true cubic over the range $I=[x_1,x_n]$ of the training inputs, $D_2=\big(\frac1{|I|}\int_I(f-q)^2\big)^{1/2}$ with $q(x)=3x^3-2x$, of the minimal solution $\btheta^*=(A^{\top}A+\lambda I)^{-1}A^{\top}y$ of Eq.~\eqref{eq:ridge-square}, for four values of the weight decay $\lambda$. \panel{G} The square of the smallest singular value of the fitting matrix, $\sigma_{\min}(A)^2$; the horizontal lines mark the three regularization levels of \panel{F}. The amplification of label noise into the fit is governed by $1/(\sigma_{\min}(A)^2+\lambda)$, so each curve in \panel{F} peaks where $\sigma_{\min}(A)^2$ falls below its $\lambda$.}
  \label{fig:main}
\end{figure}

\section{The statistical mechanics of the loss function}
\subsection{Microstates and the fundamental postulate}
\label{sec:postulate}
Statistical mechanics is the science of observing nature at large scales~\citep{reif1965fundamentals}. Just as we may know the energy of a gas but are not concerned with the coordinates and velocity of every molecule, we may know the training loss of a model but are not concerned with the specific numerical values of the model's parameters. Any set of trained weights which produces a low test loss generally makes the user happy. However, we can reason \textit{a posteriori} about the properties of the minima and how the double descent phenomenon arises, given some assumptions on the definition of \textit{training for a sufficient length of time}, i.e. \textbf{ergodicity}. \par Say that we flip ten coins and only know that seven are heads. A \textit{microstate} in this context is defined as the specific sequence of coin flip results, e.g. HTHTHHHTHH. Since we are given only the information about the total number of heads, multiple sequences (microstates) are possible, i.e. $\binom{10}{7}=\binom{10}{3}=120$. We may then define the \textbf{macrostate} $M$ of ``sequences of 10 coin flips with 7 heads" as the set of these 120 sequences, and the number of heads is known as an \textbf{observable} $O$. The number of such microstates in a macrostate is known as the multiplicity $\Omega(M)$. The \textbf{entropy} then is defined as $S\sim\log\Omega$ (we omit the Boltzmann factor $k_B$).
\begin{tcolorbox}[enhanced,colback=blue!4,colframe=blue!40!black,arc=3pt,boxrule=0.6pt,left=5pt,right=5pt,top=3pt,bottom=3pt,before
skip=4pt,after skip=4pt] \textbf{Fundamental postulate of statistical mechanics.} If all that is known about a physical system is its macroscopic observable, then every microstate consistent with that macrostate is equally probable to be the current physical state of the system.
\end{tcolorbox}
In the context of optimization, a \emph{microstate} is a specific choice of numerical parameters $\btheta$ for a model $M$, the \textit{observable} is the training loss $L(M(\btheta))$ which we will think of as the total energy $E$ of the system, for example the least squares objective $L(M(\btheta))=\frac1{2n}\sum_i(M(x_i;\btheta)-y_i)^2$\footnote{We include an arbitrary factor of $\frac12$ to interpret the loss as potential energy.}, and the \textit{macrostate} is the set of all parameters for which the model achieves this loss.
\subsection{Ergodicity}
\label{sec:ergodicity}
During model training, stochastic gradient descent~\citep{robbins1951stochastic} or another stochastic gradient-based method~\citep{kingma2015adam,loshchilov2019decoupled} is used to move from higher to lower values of the loss function. When the training error levels off, training is assumed to have reached an equilibrium. During training, a trajectory of parameters $\{\btheta_1,\btheta_2, \cdots\}$ are visited. A particular training trajectory $\{\btheta_i\}$ is called \textbf{ergodic} if it visits a representative sample of microstates, such that computing any macroscopic observable based on that sample yields the same value as if we performed the analytic sample across all solutions.
\par An ergodic trajectory that has settled at a training loss $E$ therefore samples the macrostate at that loss, the level set $\{\btheta: L(\btheta)=E\}$. In the polynomial fitting example this set is explicit. For $d\ge n$ the loss can reach $E=0$ exactly, the interpolating regime, and the macrostate is the $(d-n)$-dimensional set of interpolants; for $E>0$ it is a contour of dimension $d-1$. What the analysis needs is that the trajectory samples this set without preference for any part of it, which we take as an assumption:
\begin{tcolorbox}[enhanced,colback=blue!4,colframe=blue!40!black,arc=3pt,boxrule=0.6pt,left=5pt,right=5pt,top=3pt,bottom=3pt,before
	skip=4pt,after skip=4pt] \textbf{Ergodicity of model training.} Training a randomly-initialized model $M$ with $d$ parameters to a specified tolerance (e.g. $L(M(\btheta))=E$) produces a microstate $\btheta_d$ which is drawn uniformly from the corresponding macrostate.
\end{tcolorbox}
With this assumption in place, there is an analogy between coin flips and model training (Table~\ref{tab:analogy} in Appendix~\ref{app:coins}).

\subsection{The Boltzmann distribution and equipartition}
\label{sec:boltzmann}
The fundamental postulate weights every microstate of a fixed energy equally. Training however, does not fix the energy: the loss fluctuates from step to step. Each mini-batch reports a slightly different gradient, and the mini-batch gradient is well modeled as the full gradient plus Gaussian noise~\citep{mandt2017stochastic}. The most accurate description of training is a probability distribution over energies, and the fundamental postulate determines that this distribution must be the \textbf{Boltzmann distribution} (Appendix~\ref{app:boltzmann}):
\begin{equation}
  P(s)=\frac{e^{-E_s/T}}{Z},\qquad Z=\sum_{s}e^{-E_s/T},
  \label{eq:boltzmann}
\end{equation}
where the sum, or integral, runs over all microstates of the system regardless of energy, and $Z$ normalizes the probabilities~\citep{reif1965fundamentals}.
\par In the context of optimization, the energy is the training loss. Training for a finite time regularizes the problem, and is equivalent to weight decay of strength $\lambda\propto1/t$~\citep{ali2019continuous}, so we take the energy to be the loss with weight decay,
\begin{equation}
  L_\lambda(\btheta)=L(\btheta)+\tfrac{\lambda}{2}\|\btheta\|^2,\qquad\lambda>0.
  \label{eq:ridge-loss}
\end{equation}
For the least squares problem in Fig.~\ref{fig:main}, $L(\btheta)=\tfrac12\|A\btheta-y\|^2$ and $A$ is the fitting matrix of the basis functions at the training inputs Eq.~\eqref{eq:design}. The probability density of observing a candidate parameter $\btheta$ during training given the observed labels is then
\begin{equation}
  P(\btheta\mid y)\;\propto\;\exp\Big[-\frac{L_\lambda(\btheta)}{T}\Big],
  \label{eq:likelihood}
\end{equation}
with $T$ determined by the properties of the optimization algorithm.
\par For least squares, the loss function with weight decay is a paraboloid. Expanding the norms and completing the square,
\begin{align}
  L_\lambda(\btheta)&=\tfrac12\,\btheta^{\top}A^{\top}A\btheta-y^{\top}A\btheta+\tfrac12\|y\|^2+\tfrac{\lambda}{2}\,\btheta^{\top}\btheta
  \\&=\tfrac12\,\btheta^{\top}M\btheta-y^{\top}A\btheta+\tfrac12\|y\|^2,\qquad M=A^{\top}A+\lambda I\\
  &=\tfrac12\,(\btheta-\btheta^*)^{\top}M\,(\btheta-\btheta^*)+E_{\min},\qquad
  \btheta^*=M^{-1}A^{\top}y,\quad E_{\min}=\tfrac12\|y\|^2-\tfrac12\,y^{\top}AM^{-1}A^{\top}y,
  \label{eq:ridge-square}
\end{align}
where the third line follows from $\tfrac12(\btheta-\btheta^*)^{\top}M(\btheta-\btheta^*)=\tfrac12\btheta^{\top}M\btheta-\btheta^{\top}M\btheta^*+\tfrac12\btheta^{*\top}M\btheta^*$ and $M\btheta^*=A^{\top}y$. The matrix $M$ is symmetric positive definite for $\lambda>0$, so the loss has its minimum $E_{\min}=L_\lambda(\btheta^*)$ at $\btheta^*$, the ridge regression fit, and each contour $L_\lambda=E$ is the ellipsoid $(\btheta-\btheta^*)^{\top}M(\btheta-\btheta^*)=2(E-E_{\min})$. Any $C^\omega$ loss\footnote{We require that the loss function be real analytic ($C^\omega$) and not simply smoothly differentiable ($C^\infty$) to avoid pathologic functions such as $e^{-1/x}$, whose Taylor series does not converge to the function itself at $x=0$.} has this form to leading order near a minimum $\btheta^*$, with the Hessian $H$ in place of $M$\footnote{By definition, at a critical point of a function, its gradient is the zero vector $\mathbf{0}$, and the point curvature is positive} (Fig.~\ref{fig:landscape}). \begin{figure}[!t]
  \centering
  \includegraphics[width=\textwidth]{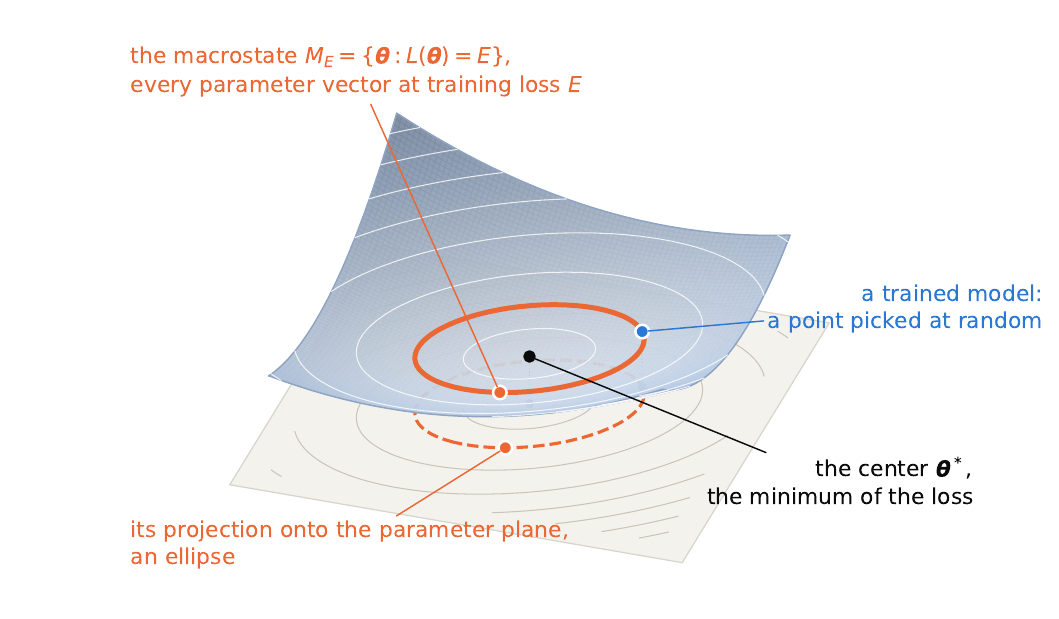}
  \caption{The fundamental postulate. The training loss $L(\btheta)$ as a landscape over two parameters $\btheta=(\theta_1,\theta_2)$, with the contour $L(\btheta)=E$ drawn on the surface (orange): it is the macrostate $M_E$, every parameter vector at training loss $E$, and its projection onto the parameter plane is an ellipse (dashed). By the fundamental postulate, a converged training run limited to a particular fixed energy visits every point of the contour of that energy equally often, so a trained model at that energy is a point picked at random from it (blue). The center $\btheta^*$ (black) is the local minimum.}
  \label{fig:landscape}
\end{figure}
\par A key consequence of the Boltzmann distribution is \textbf{the equipartition theorem}. Near a local energy minimum, the loss is quadratic, and under the Boltzmann weight each of the $d$ principal directions carries the same share $T/2$ of the loss above the minimum (Appendix~\ref{app:equipartition}), so
\begin{equation}
  \mathbb E[L]=E_{\min}+\frac{Td}{2}.
  \label{eq:equipartition}
\end{equation}
\par Near the minimum, $L=E_{\min}+\tfrac12\|w\|^2$ in the coordinates $w=H^{1/2}(\btheta-\btheta^*)$ (Eq.~\eqref{eq:ridge-square} for least squares, with $H=M$), and the contour at $E=E_{\min}+\Delta E$ is the sphere $\|w\|^2=2\Delta E$, on which the fundamental postulate makes $w$ uniform. A uniform point on a sphere has $\mathbb E[w_i^2]=2\Delta E/d$ for every coordinate, so each of the $d$ directions carries the same share $\Delta E/d$ of the excess loss. For least squares, the data term $\tfrac12\|A\btheta-y\|^2$ changes only along the $n$ directions of the row space of $AM^{-1/2}$, which we call seen by the data, and not along the $d-n$ directions of its null space, which we call unseen. The seen directions, which alone determine the fit at the training inputs, therefore carry
\begin{equation}
  \mathbb E_E\Big[\tfrac12\|w_s\|^2\Big]=\frac{n}{d}\,\Delta E,
  \label{eq:heatbath}
\end{equation}
which falls as $1/d$. At a fixed training loss $E=E_{\min}+\Delta E$, additional parameters beyond the interpolant $d>n$ therefore result in cooling of the temperature of the training algorithm along the seen directions.

\subsection{The path integral over curves}
\label{sec:pathintegral}
Read as a statement about curves rather than parameters, Eq.~\eqref{eq:likelihood} is a path integral~\citep{feynman1965quantum}. The model maps each parameter vector to a curve $f(\cdot;\btheta)$, so the sum over microstates is a sum over every curve the model can draw, each weighted by $e^{-L_\lambda/T}$, and a prediction is the average over all of them,
\begin{equation}
  Z=\int d\btheta\;e^{-L(\btheta)/T}\,e^{-\lambda\|\btheta\|^2/2T},\qquad
  \langle f(x)\rangle=\frac1Z\int d\btheta\;f(x;\btheta)\,e^{-L(\btheta)/T}\,e^{-\lambda\|\btheta\|^2/2T}.
  \label{eq:pathintegral}
\end{equation}
Under Eq.~\eqref{eq:heatbath} the temperature is $T_{\rm eff}\propto\Delta E/d$, and the variance of the coefficients becomes $2\Delta E/(\lambda d)$, shrinking with every added parameter. The \textbf{principle of least action} states that when $T\rightarrow 0$, the primary contribution to $Z$ becomes that of the \textbf{stationary path}, where $L_\lambda(\btheta)$ is minimized, and thus the exponential is maximized~(\cite{feynman1965quantum}).
\par In general, the Boltzmann distribution is equivalent to the overdamped Langevin equation
\begin{equation}
  \dot\btheta=-\nabla L(\btheta)-\lambda\btheta+\sqrt{2T}\,\xi(t),
  \label{eq:langevin}
\end{equation}
with $\xi$ white noise of unit strength: gradient descent on the loss with weight decay, driven by isotropic noise of temperature $T$, the stochastic gradient Langevin dynamics of \citet{welling2011bayesian}. The density $P(\btheta,t)$ of a cloud of runs obeys the Fokker--Planck equation~\citep{risken1996fokker}, in which the drift $-\nabla L_\lambda$ transports probability and the noise diffuses it with coefficient $T$,
\begin{equation}
  \partial_tP=\nabla\cdot\big(P\,\nabla L_\lambda\big)+T\,\nabla^2P=\nabla\cdot J,\qquad J=P\,\nabla L_\lambda+T\,\nabla P .
  \label{eq:fokker-planck}
\end{equation}
A stationary density with no probability current, $J=0$, satisfies $T\,\nabla P=-P\,\nabla L_\lambda$, that is
\begin{equation}
  \nabla\ln P=-\frac{\nabla L_\lambda}{T},\qquad\text{so}\qquad P(\btheta)=\frac1Z\,e^{-L_\lambda(\btheta)/T},
  \label{eq:stationary}
\end{equation}
which is Eq.~\eqref{eq:likelihood}, and it is normalizable exactly when $Z$ of Eq.~\eqref{eq:pathintegral} converges: for $d>n$ this requires $\lambda>0$, since without the weight decay the loss is flat along the $d-n$ unseen directions.
\par For a quadratic loss, $e^{-L_\lambda/T}=e^{-E_{\min}/T}\exp[-\tfrac12(\btheta-\btheta^*)^{\top}(H/T)(\btheta-\btheta^*)]$ is the Gaussian of mean $\btheta^*$ and covariance $TH^{-1}$, with $Z=(2\pi T)^{d/2}(\det H)^{-1/2}e^{-E_{\min}/T}$. For a model linear in its parameters, $f(x)=p(x)^{\top}\btheta$ with $p(x)$ the vector of basis functions at $x$, the curve is a Gaussian process with mean $p(x)^{\top}\btheta^*$ and covariance $T\,p(x)^{\top}H^{-1}p(x')$~\citep{rasmussen2006gaussian}. For least squares with weight decay, $H=M$ and the mean curve is the ridge regression fit $f^*(x)=p(x)^{\top}\btheta^*$.
Its two limits are the two ends of the path integral. As $T\to0$ the noise vanishes and the parameters descend to the stationary point $\nabla L(\btheta)+\lambda\btheta=0$. For least squares, $L(\btheta)=\tfrac12\|A\btheta-y\|^2=\tfrac12(A\btheta-y)^{\top}(A\btheta-y)$ has gradient $\nabla L(\btheta)=A^{\top}(A\btheta-y)$, so the stationary point solves
\begin{equation}
  A^{\top}(A\btheta-y)+\lambda\btheta=0,\qquad\text{that is}\qquad(A^{\top}A+\lambda I)\,\btheta=A^{\top}y,\qquad\btheta=\btheta^*=M^{-1}A^{\top}y,
  \label{eq:stationary-point}
\end{equation}
mirroring Eq.~\eqref{eq:ridge-square}. It is a minimum because the Hessian $M=A^{\top}A+\lambda I$ is positive definite for $\lambda>0$.
\par As $T\to\infty$, we perform a change of variables $\btheta=\sqrt{T/\lambda}\,\boldsymbol\eta$; for least squares $\nabla L(\btheta)=A^{\top}A\btheta-A^{\top}y$, and Eq.~\eqref{eq:langevin} becomes
\begin{equation}
  \dot{\boldsymbol\eta}=-(A^{\top}A+\lambda I)\boldsymbol\eta+\sqrt{\lambda/T}\,A^{\top}y+\sqrt{2\lambda}\,\xi(t)\;\xrightarrow[T\to\infty]{}\;-M\boldsymbol\eta+\sqrt{2\lambda}\,\xi(t).
  \label{eq:langevin-hot}
\end{equation}
\par The $T\rightarrow \infty$ limit Eq.~\eqref{eq:langevin-hot} is the Ornstein--Uhlenbeck process~\citep{uhlenbeck1930theory}, resulting in the motion $\btheta(t)-\btheta^*=e^{-Ht}(\btheta_0-\btheta^*)+\sqrt{2T}\int_0^te^{-H(t-s)}\,dW_s$. This solution is Gaussian at every $t$, with mean $\btheta^*+e^{-Ht}(\btheta_0-\btheta^*)$ and covariance $2T\int_0^te^{-2Hv}dv=TH^{-1}(I-e^{-2Ht})$, which tends to the stationary law as $t\to\infty$~\citep[Chapter~4]{pavliotis2014stochastic}.
For the value at $x$, with $\tilde p_i(x)$ the components of $p(x)$ along the eigenvectors of $M$ and $\mu_i$ its eigenvalues,
\begin{equation}
  \langle f(x)^2\rangle_t=\langle f(x)\rangle_t^2+T\sum_i\tilde p_i(x)^2\,\frac{1-e^{-2\mu_it}}{\mu_i},\qquad
  \langle f(x)\rangle_t=f^*(x)+p(x)^{\top}e^{-Mt}(\btheta_0-\btheta^*).
  \label{eq:ou-second-moment}
\end{equation}
At stationarity $\langle f(x)^2\rangle=f^*(x)^2+T\,p(x)^{\top}M^{-1}p(x)$. At a finite time the unseen directions, $\mu_i=\lambda$, have reached only $T(1-e^{-2\lambda t})/\lambda\approx2Tt$ of their stationary variance $T/\lambda$: the run behaves as if the stiffness were capped at $1/(2t)$, the $\lambda\propto1/t$ of Eq.~\eqref{eq:ridge-loss}.
\begin{theorem}[Finite training time is weight decay, Appendix~\ref{app:proofs}]\label{thm:time}
For least squares with weight decay $\lambda\ge0$, the run of Eq.~\eqref{eq:langevin} from $\btheta_0=0$ has, at time $t$,
\[
  \mathrm{Cov}[\btheta(t)]\;\preceq\;2T\,\big(M+\tfrac{1}{2t}I\big)^{-1},\qquad
  \mathbb E[\btheta(t)]=(I-e^{-Mt})M^{-1}A^{\top}y,
\]
and the mean agrees with the ridge fit $(M+\tfrac1tI)^{-1}A^{\top}y$ at weight decay $\lambda+1/t$ along every eigenvector of $M$ up to a factor between $1$ and $1.3$.
\end{theorem}
\noindent By Theorem~\ref{thm:time}, a run of length $t$ has the fluctuations of a run with weight decay $\lambda+1/(2t)$, up to a factor two, and the mean fit of a run with weight decay $\lambda+1/t$, up to a factor $1.3$, in agreement with the continuous-time analysis of early stopping by \citet{ali2019continuous}; the finite training time caps the size of the fit whatever the conditioning of the data. Late stopping removes the cap and leaves the ridge fit at the explicit $\lambda$, whose quality depends on the conditioning of the data: the collapse of $\sigma_{\min}(A)$ at $d=n$ in Fig.~\ref{fig:main}\panel{G} and its recovery beyond.
\subsection{Measuring the size of the fitted functions}
\par The Legendre polynomials are orthogonal, $\int_{-1}^{1}P_kP_l\,dx=\frac{2}{2k+1}\delta_{kl}$~\citep[Eq.~(4.3.3) with $\alpha=\beta=0$]{szego1975orthogonal}, so that
\begin{equation}
  \int_{-1}^{1}f^2\,dx=\sum_{k<d}\frac{2}{2k+1}\,\theta_k^2\;\le\;2\|\btheta\|^2,
  \label{eq:parseval}
\end{equation}
therefore the least-action curve through the data is the smallest one in the $L^2$ sense. For $d\ge n$ and distinct inputs, $AA^{\top}$ is invertible and as $\lambda\to0$ the least-action curve is the interpolant of least norm, $\btheta^*=A^{\top}(AA^{\top})^{-1}y$ (Theorem~\ref{thm:regimes} in Appendix~\ref{app:stability}), whose energy $\tfrac12\|\btheta^*\|^2=\tfrac12y^{\top}(AA^{\top})^{-1}y$ changes with every added coefficient. We now show that this energy is nonincreasing.
\begin{theorem}[The energy of the interpolant, Appendix~\ref{app:proofs}]\label{thm:norm}
Let $d\ge n$, let $p_d=(P_d(x_i))_{i\le n}$ be the values at the inputs of the basis function added when $d$ becomes $d+1$, and let $\alpha_d=(AA^{\top})^{-1}y$, so that $\btheta^*=A^{\top}\alpha_d$. Then
\begin{equation}
  \|\btheta^{*(d+1)}\|^2=\|\btheta^{*(d)}\|^2-\frac{(p_d^{\top}\alpha_d)^2}{1+p_d^{\top}(AA^{\top})^{-1}p_d}:
  \label{eq:norm-drop}
\end{equation}
the energy of the interpolant is nonincreasing in $d$, and strictly decreasing unless $p_d^{\top}\alpha_d=0$, which for given inputs holds only on a hyperplane of labels $y$.
\end{theorem}
\noindent Thus, each new basis function with nonzero overlap with the training data lowers the energy.
\section{Discussion}
In this work, we have described a theoretical basis for how overparametrization past the interpolation threshold leads to increased regularization of the optimization problem. Schematically, there is leakage of energy into the zero modes of higher dimensions, assuming that training has reached thermal equilibrium. Interpreting training as a functional Eq.~\eqref{eq:pathintegral} provides physical intuition for what these extra dimensions look like. Note that while we examined 1D Legendre polynomials here, the same discussions presumably apply to higher dimensional data in the language of partial differential equations and representation theory. The Legendre polynomials are the eigenfunctions of a Sturm--Liouville problem, $\frac{d}{dx}[(1-x^2)P_k']=-k(k+1)P_k$, and Sturm--Liouville theory is what makes them a complete orthogonal basis~\citep[Chapters~V and~VII]{courant1953methods}. Higher dimensional analogues, such as the spherical harmonics as eigenfunctions of the Laplacian on the sphere, form a complete orthogonal basis in the same way, and thus may be applicable to higher dimensional optimization problems~\citep{atkinson2012spherical}. However, the analogy of model training with particle motion remains unchanged. Fortunately, a large and rigorous corpus of knowledge exists in statistical mechanics, operator theory, and beyond to crystallize this analogy.
\par We have not directly shown that the phenomena discussed decrease the test loss. However, it is well known that regularization during training decreases the test loss and improves the ability of the model to learn the underlying data~\citep{krogh1992simple,pascanu2013difficulty,srivastava2014dropout,ioffe2015batch,ba2016layer,he2016deep}. The idea that solutions which generalize well have low complexity as measured by sharpness, curvature, and parameter norms is an expression of the \textbf{manifold hypothesis}, that real world data lie along low-dimensional smooth manifolds~\citep{fefferman2016testing}. Double descent appears to be another tool for taming the wilderness of possible functions.
\par There are a few interesting observations we would like to make. First, the loss landscape of neural networks is riddled with local minima. These minima exhibit discrete symmetries (for example, permuting the neurons) and continuous symmetries (for example, a change of basis for the linear output of a layer of neurons). The topic of spontaneously broken symmetries appears in physics and is fundamental to our understanding of how some particles acquire mass~\citep{englert1964broken,higgs1964broken}, which is similar to weight regularization, in which a ``mass'' term $\lambda \|\theta\|^2$ breaks the continuous symmetry present in the overinterpolated regime so that new minima form. In the case of model training, we explicitly add this term, whereas in particle physics, the universe spontaneously breaks the symmetry by picking a vacuum state.
\par Second, if we take Fig.~\ref{fig:landscape} literally, at thermal equilibrium around a local minimum, the parameters are sampled in an ellipsoid around that minimum. By mirror symmetry about the principal axes of that ellipsoid, the picture predicts that $\langle\btheta\rangle=\btheta^*$. Exponential moving averages of parameters~\citep{morales2024exponential}, and iterate averaging before them~\citep{polyak1992acceleration}, exploit this behavior, discounting the far past of the training trajectory, while weighting recently visited states near equally, which may explain why these models generalize well, and why more general mixture-of-experts models~\citep{jacobs1991adaptive,shazeer2017outrageously} exhibit similar improvements.
\par Finally, diffusion and flow matching models~\citep{holderrieth2025introduction} clearly fit in the same physical framework as model training, with discussion of Brownian motion and the It\^o calculus. It may be worthwhile to translate findings between these fields of research.
\section{Conclusion}
The training trajectory produced by stochastic gradient-based methods can be thought of as a particle wandering over the energy landscape with some induced temperature $T$, and so can be treated with statistical mechanics. Because real world machine learning algorithms start at some initial point and have only finite time to diffuse in the energy landscape, there is effectively weight decay regularization. By increasing the degrees of freedom in this landscape, the equipartition theorem distributes energy among the quadratic coordinates in multiples of $\frac T2$. Increasing the number of degrees of freedom $d$ (which are made quadratic due to weight decay regularization) thus effectively decreases the temperature for a given energy $E$, driving the Boltzmann distribution toward the stationary path. Furthermore, increasing $d$ can only lower the energy of the stationary path (Theorem~\ref{thm:norm}) and thus the resulting $L^2$ norm Eq.~\eqref{eq:parseval}, therefore reducing the likelihood at fixed $E$ that the sampled solution has large $L^2$ norm. Thus, increasing $d$ beyond the interpolation threshold is a form of weight decay regularization as well.
\section*{AI disclosure}
The author used Claude (Anthropic; Claude Opus 5 through Claude Code) as a research and writing assistant. Under the author's direction it drafted and revised text, derived and typeset the theorems and their proofs, wrote the Lean~4 verification of the theorems, wrote the Python scripts that produce every figure, and ran the numerical experiments. The author's initial idea was to explore double descent through the lens of statistical mechanics, in particular to bound the norms of solutions near local minima, as a function of $d$. The author noted that there is symmetry among minima for neural networks by permuting neurons in a hidden layer. From there, the author was reminded of the spontaneous symmetry breaking in physics, of the path integral, and finally of the connection between the path integral and the Boltzmann distribution.
\subsubsection*{Conflict of interest statement}
The author is a 2026--2027 Doximity AI Fellow, and receives items and services of minimal monetary value from Doximity, Inc in exchange for consulting services.
\bibliography{refs}

\begin{thebibliography}{39}
\providecommand{\natexlab}[1]{#1}
\providecommand{\url}[1]{\texttt{#1}}
\expandafter\ifx\csname urlstyle\endcsname\relax
  \providecommand{\doi}[1]{doi: #1}\else
  \providecommand{\doi}{doi: \begingroup \urlstyle{rm}\Url}\fi

\bibitem[Ali et~al.(2019)Ali, Kolter, and Tibshirani]{ali2019continuous}
Alnur Ali, J~Zico Kolter, and Ryan~J Tibshirani.
\newblock A continuous-time view of early stopping for least squares
  regression.
\newblock In \emph{Proceedings of the 22nd International Conference on
  Artificial Intelligence and Statistics (AISTATS)}, volume~89 of
  \emph{Proceedings of Machine Learning Research}, pp.\  1370--1378, 2019.
\newblock URL \url{https://proceedings.mlr.press/v89/ali19a.html}.

\bibitem[Atkinson \& Han(2012)Atkinson and Han]{atkinson2012spherical}
Kendall Atkinson and Weimin Han.
\newblock \emph{Spherical Harmonics and Approximations on the Unit Sphere: An
  Introduction}, volume 2044 of \emph{Lecture Notes in Mathematics}.
\newblock Springer, Berlin, 2012.
\newblock \doi{10.1007/978-3-642-25983-8}.

\bibitem[Ba et~al.(2016)Ba, Kiros, and Hinton]{ba2016layer}
Jimmy~Lei Ba, Jamie~Ryan Kiros, and Geoffrey~E Hinton.
\newblock Layer normalization, 2016.
\newblock \href{https://arxiv.org/abs/1607.06450}{arXiv:1607.06450}.

\bibitem[Bartlett et~al.(2020)Bartlett, Long, Lugosi, and
  Tsigler]{bartlett2020benign}
Peter~L Bartlett, Philip~M Long, G{\'a}bor Lugosi, and Alexander Tsigler.
\newblock Benign overfitting in linear regression.
\newblock \emph{Proceedings of the National Academy of Sciences}, 117\penalty0
  (48):\penalty0 30063--30070, 2020.
\newblock \doi{10.1073/pnas.1907378117}.

\bibitem[Belkin et~al.(2019)Belkin, Hsu, Ma, and Mandal]{belkin2019reconciling}
Mikhail Belkin, Daniel Hsu, Siyuan Ma, and Soumik Mandal.
\newblock Reconciling modern machine-learning practice and the classical
  bias--variance trade-off.
\newblock \emph{Proceedings of the National Academy of Sciences}, 116\penalty0
  (32):\penalty0 15849--15854, 2019.
\newblock \doi{10.1073/pnas.1903070116}.

\bibitem[Courant \& Hilbert(1953)Courant and Hilbert]{courant1953methods}
Richard Courant and David Hilbert.
\newblock \emph{Methods of Mathematical Physics}, volume~1.
\newblock Interscience, New York, 1953.
\newblock \doi{10.1002/9783527617210}.

\bibitem[Englert \& Brout(1964)Englert and Brout]{englert1964broken}
Fran\c{c}ois Englert and Robert Brout.
\newblock Broken symmetry and the mass of gauge vector mesons.
\newblock \emph{Physical Review Letters}, 13\penalty0 (9):\penalty0 321--323,
  1964.
\newblock \doi{10.1103/PhysRevLett.13.321}.

\bibitem[Fefferman et~al.(2016)Fefferman, Mitter, and
  Narayanan]{fefferman2016testing}
Charles Fefferman, Sanjoy Mitter, and Hariharan Narayanan.
\newblock Testing the manifold hypothesis.
\newblock \emph{Journal of the American Mathematical Society}, 29\penalty0
  (4):\penalty0 983--1049, 2016.
\newblock \doi{10.1090/jams/852}.

\bibitem[Feynman \& Hibbs(1965)Feynman and Hibbs]{feynman1965quantum}
Richard~P Feynman and Albert~R Hibbs.
\newblock \emph{Quantum Mechanics and Path Integrals}.
\newblock McGraw-Hill, New York, 1965.

\bibitem[Geiger et~al.(2020)Geiger, Jacot, Spigler, Gabriel, Sagun, d'Ascoli,
  Biroli, Hongler, and Wyart]{geiger2020scaling}
Mario Geiger, Arthur Jacot, Stefano Spigler, Franck Gabriel, Levent Sagun,
  St{\'e}phane d'Ascoli, Giulio Biroli, Cl{\'e}ment Hongler, and Matthieu
  Wyart.
\newblock Scaling description of generalization with number of parameters in
  deep learning.
\newblock \emph{Journal of Statistical Mechanics: Theory and Experiment},
  2020\penalty0 (2):\penalty0 023401, 2020.
\newblock \doi{10.1088/1742-5468/ab633c}.

\bibitem[Hastie et~al.(2022)Hastie, Montanari, Rosset, and
  Tibshirani]{hastie2022surprises}
Trevor Hastie, Andrea Montanari, Saharon Rosset, and Ryan~J Tibshirani.
\newblock Surprises in high-dimensional ridgeless least squares interpolation.
\newblock \emph{Annals of Statistics}, 50\penalty0 (2):\penalty0 949--986,
  2022.
\newblock \doi{10.1214/21-AOS2133}.

\bibitem[He et~al.(2016)He, Zhang, Ren, and Sun]{he2016deep}
Kaiming He, Xiangyu Zhang, Shaoqing Ren, and Jian Sun.
\newblock Deep residual learning for image recognition.
\newblock In \emph{IEEE Conference on Computer Vision and Pattern Recognition
  (CVPR)}, pp.\  770--778, 2016.
\newblock \doi{10.1109/CVPR.2016.90}.

\bibitem[Higgs(1964)]{higgs1964broken}
Peter~W Higgs.
\newblock Broken symmetries and the masses of gauge bosons.
\newblock \emph{Physical Review Letters}, 13\penalty0 (16):\penalty0 508--509,
  1964.
\newblock \doi{10.1103/PhysRevLett.13.508}.

\bibitem[Holderrieth \& Erives(2025)Holderrieth and
  Erives]{holderrieth2025introduction}
Peter Holderrieth and Ezra Erives.
\newblock An introduction to flow matching and diffusion models.
\newblock Lecture notes for MIT 6.S184, Generative AI with Stochastic
  Differential Equations, \url{https://diffusion.csail.mit.edu/}, 2025.
\newblock \href{https://arxiv.org/abs/2506.02070}{arXiv:2506.02070}.

\bibitem[Ioffe \& Szegedy(2015)Ioffe and Szegedy]{ioffe2015batch}
Sergey Ioffe and Christian Szegedy.
\newblock Batch normalization: Accelerating deep network training by reducing
  internal covariate shift.
\newblock In \emph{Proceedings of the 32nd International Conference on Machine
  Learning (ICML)}, volume~37 of \emph{Proceedings of Machine Learning
  Research}, pp.\  448--456, 2015.
\newblock URL \url{https://proceedings.mlr.press/v37/ioffe15.html}.

\bibitem[Jacobs et~al.(1991)Jacobs, Jordan, Nowlan, and
  Hinton]{jacobs1991adaptive}
Robert~A Jacobs, Michael~I Jordan, Steven~J Nowlan, and Geoffrey~E Hinton.
\newblock Adaptive mixtures of local experts.
\newblock \emph{Neural Computation}, 3\penalty0 (1):\penalty0 79--87, 1991.
\newblock \doi{10.1162/neco.1991.3.1.79}.

\bibitem[Kingma \& Ba(2015)Kingma and Ba]{kingma2015adam}
Diederik~P Kingma and Jimmy Ba.
\newblock Adam: A method for stochastic optimization.
\newblock In \emph{International Conference on Learning Representations
  (ICLR)}, 2015.
\newblock \href{https://arxiv.org/abs/1412.6980}{arXiv:1412.6980}.

\bibitem[Krogh \& Hertz(1992)Krogh and Hertz]{krogh1992simple}
Anders Krogh and John~A Hertz.
\newblock A simple weight decay can improve generalization.
\newblock In \emph{Advances in Neural Information Processing Systems},
  volume~4, pp.\  950--957. Morgan Kaufmann, 1992.
\newblock URL
  \url{https://proceedings.neurips.cc/paper/1991/hash/8eefcfdf5990e441f0fb6f3fad709e21-Abstract.html}.

\bibitem[Li \& Goldenfeld(2026)Li and Goldenfeld]{li2026broken}
Chan Li and Nigel Goldenfeld.
\newblock Broken ergodicity and the violation of the fluctuation-dissipation
  theorem lead to generalization beyond overfitting in machine learning, 2026.
\newblock \href{https://arxiv.org/abs/2607.04135}{arXiv:2607.04135}.

\bibitem[Loshchilov \& Hutter(2019)Loshchilov and
  Hutter]{loshchilov2019decoupled}
Ilya Loshchilov and Frank Hutter.
\newblock Decoupled weight decay regularization.
\newblock In \emph{International Conference on Learning Representations
  (ICLR)}, 2019.
\newblock \href{https://arxiv.org/abs/1711.05101}{arXiv:1711.05101}.

\bibitem[Mandt et~al.(2017)Mandt, Hoffman, and Blei]{mandt2017stochastic}
Stephan Mandt, Matthew~D Hoffman, and David~M Blei.
\newblock Stochastic gradient descent as approximate {B}ayesian inference.
\newblock \emph{Journal of Machine Learning Research}, 18\penalty0
  (134):\penalty0 1--35, 2017.

\bibitem[Mei \& Montanari(2022)Mei and Montanari]{mei2022generalization}
Song Mei and Andrea Montanari.
\newblock The generalization error of random features regression: Precise
  asymptotics and the double descent curve.
\newblock \emph{Communications on Pure and Applied Mathematics}, 75\penalty0
  (4):\penalty0 667--766, 2022.
\newblock \doi{10.1002/cpa.22008}.

\bibitem[Morales-Brotons et~al.(2024)Morales-Brotons, Vogels, and
  Hendrikx]{morales2024exponential}
Daniel Morales-Brotons, Thijs Vogels, and Hadrien Hendrikx.
\newblock Exponential moving average of weights in deep learning: Dynamics and
  benefits.
\newblock \emph{Transactions on Machine Learning Research}, 2024.
\newblock \href{https://arxiv.org/abs/2411.18704}{arXiv:2411.18704}.

\bibitem[Nakkiran et~al.(2019)Nakkiran, Kaplun, Bansal, Yang, Barak, and
  Sutskever]{nakkiran2019blog}
Preetum Nakkiran, Gal Kaplun, Yamini Bansal, Tristan Yang, Boaz Barak, and Ilya
  Sutskever.
\newblock Deep double descent.
\newblock Windows on Theory blog,
  \url{https://windowsontheory.org/2019/12/05/deep-double-descent/}, December
  2019.

\bibitem[Nakkiran et~al.(2020)Nakkiran, Kaplun, Bansal, Yang, Barak, and
  Sutskever]{nakkiran2020deep}
Preetum Nakkiran, Gal Kaplun, Yamini Bansal, Tristan Yang, Boaz Barak, and Ilya
  Sutskever.
\newblock Deep double descent: Where bigger models and more data hurt.
\newblock In \emph{International Conference on Learning Representations
  (ICLR)}, 2020.
\newblock \href{https://arxiv.org/abs/1912.02292}{arXiv:1912.02292}.

\bibitem[Pascanu et~al.(2013)Pascanu, Mikolov, and
  Bengio]{pascanu2013difficulty}
Razvan Pascanu, Tomas Mikolov, and Yoshua Bengio.
\newblock On the difficulty of training recurrent neural networks.
\newblock In \emph{Proceedings of the 30th International Conference on Machine
  Learning (ICML)}, volume~28 of \emph{Proceedings of Machine Learning
  Research}, pp.\  1310--1318, 2013.
\newblock URL \url{https://proceedings.mlr.press/v28/pascanu13.html}.

\bibitem[Pavliotis(2014)]{pavliotis2014stochastic}
Grigorios~A Pavliotis.
\newblock \emph{Stochastic Processes and Applications: Diffusion Processes, the
  {F}okker--{P}lanck and {L}angevin Equations}, volume~60 of \emph{Texts in
  Applied Mathematics}.
\newblock Springer, New York, 2014.
\newblock \doi{10.1007/978-1-4939-1323-7}.

\bibitem[Polyak \& Juditsky(1992)Polyak and Juditsky]{polyak1992acceleration}
Boris~T Polyak and Anatoli~B Juditsky.
\newblock Acceleration of stochastic approximation by averaging.
\newblock \emph{SIAM Journal on Control and Optimization}, 30\penalty0
  (4):\penalty0 838--855, 1992.
\newblock \doi{10.1137/0330046}.

\bibitem[Rasmussen \& Williams(2006)Rasmussen and
  Williams]{rasmussen2006gaussian}
Carl~Edward Rasmussen and Christopher K~I Williams.
\newblock \emph{Gaussian Processes for Machine Learning}.
\newblock MIT Press, 2006.
\newblock \doi{10.7551/mitpress/3206.001.0001}.

\bibitem[Reif(1965)]{reif1965fundamentals}
Frederick Reif.
\newblock \emph{Fundamentals of Statistical and Thermal Physics}.
\newblock McGraw-Hill, New York, 1965.

\bibitem[Risken(1996)]{risken1996fokker}
Hannes Risken.
\newblock \emph{The {F}okker--{P}lanck Equation: Methods of Solution and
  Applications}.
\newblock Springer, Berlin, 2nd edition, 1996.
\newblock \doi{10.1007/978-3-642-61544-3}.

\bibitem[Robbins \& Monro(1951)Robbins and Monro]{robbins1951stochastic}
Herbert Robbins and Sutton Monro.
\newblock A stochastic approximation method.
\newblock \emph{Annals of Mathematical Statistics}, 22\penalty0 (3):\penalty0
  400--407, 1951.
\newblock \doi{10.1214/aoms/1177729586}.

\bibitem[Shazeer et~al.(2017)Shazeer, Mirhoseini, Maziarz, Davis, Le, Hinton,
  and Dean]{shazeer2017outrageously}
Noam Shazeer, Azalia Mirhoseini, Krzysztof Maziarz, Andy Davis, Quoc Le,
  Geoffrey Hinton, and Jeff Dean.
\newblock Outrageously large neural networks: The sparsely-gated
  mixture-of-experts layer.
\newblock In \emph{International Conference on Learning Representations
  (ICLR)}, 2017.
\newblock \href{https://arxiv.org/abs/1701.06538}{arXiv:1701.06538}.

\bibitem[Sherman \& Morrison(1950)Sherman and Morrison]{sherman1950adjustment}
Jack Sherman and Winifred~J Morrison.
\newblock Adjustment of an inverse matrix corresponding to a change in one
  element of a given matrix.
\newblock \emph{The Annals of Mathematical Statistics}, 21\penalty0
  (1):\penalty0 124--127, 1950.
\newblock \doi{10.1214/aoms/1177729893}.

\bibitem[Spigler et~al.(2019)Spigler, Geiger, d'Ascoli, Sagun, Biroli, and
  Wyart]{spigler2019jamming}
Stefano Spigler, Mario Geiger, St{\'e}phane d'Ascoli, Levent Sagun, Giulio
  Biroli, and Matthieu Wyart.
\newblock A jamming transition from under- to over-parametrization affects
  generalization in deep learning.
\newblock \emph{Journal of Physics A: Mathematical and Theoretical},
  52\penalty0 (47):\penalty0 474001, 2019.
\newblock \doi{10.1088/1751-8121/ab4c8b}.

\bibitem[Srivastava et~al.(2014)Srivastava, Hinton, Krizhevsky, Sutskever, and
  Salakhutdinov]{srivastava2014dropout}
Nitish Srivastava, Geoffrey Hinton, Alex Krizhevsky, Ilya Sutskever, and Ruslan
  Salakhutdinov.
\newblock Dropout: A simple way to prevent neural networks from overfitting.
\newblock \emph{Journal of Machine Learning Research}, 15\penalty0
  (56):\penalty0 1929--1958, 2014.
\newblock URL \url{https://jmlr.org/papers/v15/srivastava14a.html}.

\bibitem[Szeg{\H o}(1975)]{szego1975orthogonal}
G{\'a}bor Szeg{\H o}.
\newblock \emph{Orthogonal Polynomials}, volume~23 of \emph{Colloquium
  Publications}.
\newblock American Mathematical Society, 4th edition, 1975.
\newblock \doi{10.1090/coll/023}.

\bibitem[Uhlenbeck \& Ornstein(1930)Uhlenbeck and
  Ornstein]{uhlenbeck1930theory}
George~E Uhlenbeck and Leonard~S Ornstein.
\newblock On the theory of the {B}rownian motion.
\newblock \emph{Physical Review}, 36\penalty0 (5):\penalty0 823--841, 1930.
\newblock \doi{10.1103/PhysRev.36.823}.

\bibitem[Welling \& Teh(2011)Welling and Teh]{welling2011bayesian}
Max Welling and Yee~Whye Teh.
\newblock Bayesian learning via stochastic gradient {L}angevin dynamics.
\newblock In \emph{Proceedings of the 28th International Conference on Machine
  Learning (ICML)}, pp.\  681--688, 2011.
\newblock URL \url{https://icml.cc/2011/papers/398_icmlpaper.pdf}.

\end{thebibliography}
\bibliographystyle{tmlr}

\appendix
\setcounter{theorem}{0}\renewcommand{\thetheorem}{\Alph{section}.\arabic{theorem}}\renewcommand{\theHtheorem}{\Alph{section}.\arabic{theorem}}
\section{Worked examples and derivations for Section~\ref{sec:postulate}}
\label{app:background}
\subsection{Coin flips}
\label{app:coins}
With the fundamental postulate, one may count physical states of a collection of objects, given macrostates defined by a macroscopic observable, such as the total energy $E$ or the average pressure $P$, and derive laws regarding how these quantities relate with each other, for example the ideal gas law $PV=Nk_BT$~\citep{reif1965fundamentals}.
\par With the ergodicity assumption of Section~\ref{sec:ergodicity} in place, there is an analogy between the coin flips and model training (Table~\ref{tab:analogy}).

\begin{table}[h]
\centering
\begin{tabular}{@{}>{\raggedright\arraybackslash}p{2.6cm}>{\raggedright\arraybackslash}p{5.6cm}>{\raggedright\arraybackslash}p{6.4cm}@{}}
\toprule
 & \textbf{Ten coin flips} & \textbf{Machine learning model $M$} \\
\midrule
Microstate & One sequence of outcomes, e.g.\ HTHTHHHTHH & A specific vector of parameters $\btheta$
\\
Observable & Total number of heads & Training loss $L(M(\btheta))$ \\
Macrostate & All sequences with the observed number of heads, e.g.\ the 120 sequences with 7 heads & All parameter vectors that achieve the observed loss of $E$: $\mathbf{\Theta}=\{\btheta: L(M(\btheta))=E\}$ \\
Multiplicity $\Omega$ & $\binom{10}{7}=120$ & Number/density of the set of parameters with loss $E$
\\
Fundamental postulate & Each of the 120 sequences has probability $1/120$ & All $\btheta\in \mathbf{\Theta}$ are equally likely to have been reached once training has converged\\
\bottomrule
\end{tabular}
\caption{The statistical mechanics analogy between coin flips and machine learning models.}
\label{tab:analogy}
\end{table}

\subsection{The Boltzmann distribution}
\label{app:boltzmann}
Consider a small system $\mathcal S$ in contact with a large reservoir $\mathcal R$, the two together isolated with total energy $E_{\rm tot}$ and able to exchange energy. A microstate of the whole is a pair (microstate of $\mathcal S$, microstate of $\mathcal R$), and the postulate weights every such pair with energy $E_{\rm tot}$ equally. Fix one microstate $s$ of the system, of energy $E_s$. The number of pairs in which $\mathcal S$ is in state $s$ is the number of reservoir microstates of energy $E_{\rm tot}-E_s$, that is, the multiplicity $\Omega_{\mathcal R}(E_{\rm tot}-E_s)$. Hence
\begin{equation}
  P(s)\;\propto\;\Omega_{\mathcal R}(E_{\rm tot}-E_s)=\exp\big[S_{\mathcal R}(E_{\rm tot}-E_s)\big].
  \label{eq:reservoir}
\end{equation}
The system is small, so $E_s\ll E_{\rm tot}$, and the exponent can be expanded to first order: $S_{\mathcal R}(E_{\rm tot}-E_s)=S_{\mathcal R}(E_{\rm tot})-E_s\,\partial S_{\mathcal R}/\partial E+\dots$. The first term is a constant. The derivative $\partial S_{\mathcal R}/\partial E$ is a property of the reservoir alone, and we call its reciprocal the \textbf{temperature} $T$ (again ignoring the historical factor of $k_B$). The result is Eq.~\eqref{eq:boltzmann}.
\par With the training loss as the energy, Eq.~\eqref{eq:likelihood}, a model that fits the data less well is not forbidden, only exponentially less probable, by a factor $e^{-\Delta L/T}$ per unit of extra loss. If, in addition, the coefficients are penalized by $\tfrac{\lambda}{2}\|\btheta\|^2$, the penalty adds to the energy, and the weight becomes $\exp[-L_\lambda(\btheta)/T]$ with $L_\lambda=L+\tfrac{\lambda}{2}\|\btheta\|^2$; the factor $\exp[-\lambda\|\btheta\|^2/2T]$ is a Gaussian prior of variance $s^2=T/\lambda$ on each coefficient. Zero temperature, $T\to0$, concentrates the weight on the minimizers of $L_\lambda$. Zero penalty, $\lambda\to0$, recovers the initial uniform weighting of microstates.

\subsection{Equipartition}
\label{app:equipartition}
In a neighborhood about a local energy minimum $E_{\min}$, the landscape is quadratic, $L(\btheta)=E_{\min}+\tfrac12\sum_{i=1}^{d}\mu_iu_i^2$, in coordinates $u$ along its principal curvatures $\mu_i>0$ (Eq.~\eqref{eq:ridge-square} makes this exact for least squares). Under the Boltzmann weight $e^{-L/T}$ the coordinates are independent, each with density $\propto e^{-\mu_iu_i^2/2T}$, a Gaussian of variance $T/\mu_i$, so each term of the sum has average $\tfrac12\mu_i\cdot T/\mu_i=T/2$ whatever its curvature, and
\[
  \mathbb E[L]=E_{\min}+\sum_{i=1}^{d}\frac{T}{2}=E_{\min}+\frac{Td}{2},
\]
which is Eq.~\eqref{eq:equipartition}.

\section{The Legendre fit in each regime}
\label{app:stability}
The fits of Fig.~\ref{fig:main} and the matrix $A$ of Section~\ref{sec:boltzmann} are defined as follows. We work in the Legendre basis $P_0,\dots,P_{d-1}$, the polynomials of Rodrigues' formula~\citep[Eq.~(4.3.1) with $\alpha=\beta=0$]{szego1975orthogonal}
\begin{equation}
  P_k(x)=\frac{1}{2^kk!}\,\frac{d^k}{dx^k}\big(x^2-1\big)^k,\qquad
  P_0=1,\quad P_1=x,\quad P_2=\tfrac12(3x^2-1),\quad P_3=\tfrac12(5x^3-3x),\ \dots,
  \label{eq:legendre}
\end{equation}
so that $P_k$ has degree exactly $k$ and they are orthogonal on $[-1,1]$, $\int_{-1}^{1} P_kP_l\,dx=\frac{2}{2k+1}\delta_{kl}$.\footnote{The monomial, or Taylor, basis $\{1,x,x^2,\dots\}$ spans the same polynomials but is numerically unstable: the corresponding $A$ is the Vandermonde matrix, whose condition number grows exponentially with $d$, because the monomials of high degree are all nearly zero on most of the interval and nearly equal to one another near $x=1$ and $x=-1$.} Every Legendre polynomial is bounded by one on the whole interval, $|P_k(x)|\le1$ with $P_k(1)=1$, so the columns of $A$ have entries of comparable size. We write $f(x)=\sum_k\theta_kP_k(x)=p(x)^{\top}\btheta$, so that the value $f(x)$ at any input is a linear observable of the microstate $\btheta$. The training constraints are $A\btheta=y$, one row per training point and one column per basis function,
\begin{equation}
  A=\begin{pmatrix}
    P_0(x_1) & P_1(x_1) & \cdots & P_{d-1}(x_1)\\[4pt]
    P_0(x_2) & P_1(x_2) & \cdots & P_{d-1}(x_2)\\[4pt]
    \vdots & \vdots & \ddots & \vdots\\[2pt]
    P_0(x_n) & P_1(x_n) & \cdots & P_{d-1}(x_n)
  \end{pmatrix}
  =\begin{pmatrix} 1 & x_1 & \frac12(3x_1^2-1) & \cdots\\[4pt] 1 & x_2 & \frac12(3x_2^2-1) & \cdots\\[4pt] \vdots & \vdots & \vdots & \ddots\\[2pt] 1 & x_n & \frac12(3x_n^2-1) & \cdots
  \end{pmatrix},\qquad
  \btheta=\begin{pmatrix}\theta_0\\ \theta_1\\ \vdots\\ \theta_{d-1}\end{pmatrix},\quad
  y=\begin{pmatrix}y_1\\ y_2\\ \vdots\\ y_n\end{pmatrix},
  \label{eq:design}
\end{equation}
an $n\times d$ matrix, the fitting matrix, whose $i$th row is $p(x_i)^{\top}$, so that $(A\btheta)_i=f(x_i)$.
\begin{theorem}[The fit in each regime]\label{thm:regimes}
Let the $n$ training inputs be distinct. Then $A$ has rank $\min(d,n)$, and the fit that training aims at is linear in the labels, $\btheta=Wy$ for a $d\times n$ matrix $W$ that depends only on the inputs; in each regime:
\begin{itemize}[leftmargin=4.2em, labelwidth=3.6em, labelsep=0.6em, align=left, itemsep=1pt, topsep=2pt]
\item[$d<n$:] $W=(A^{\top}A)^{-1}A^{\top}$, least squares. The macrostate $M$ at $E=0$ is empty; the training loss has the unique minimizer $\btheta=Wy$, with $E_{\min}=\tfrac12\|(I-AW)y\|^2>0$ unless $y$ lies in the column space of $A$;
\item[$d=n$:] $W=A^{-1}$. $A$ is invertible and $M=\{Wy\}$, the unique interpolant known as the Lagrange polynomial through the $n$ points;
\item[$d>n$:] $W=A^{\top}(AA^{\top})^{-1}$. $M=\btheta^*+\ker A$ is an affine subspace of dimension $d-n$, and its member of least norm is $\btheta^*=Wy$, the minimum-norm interpolant, which is where gradient descent from $\btheta=0$ converges.
\end{itemize}
\end{theorem}
\begin{proof}
\emph{The rank.} Restrict $A$ to its first $\min(d,n)$ columns and the first $\min(d,n)$ inputs; this square block is $V_{ik}=P_k(x_i)$. Since $P_k$ has degree exactly $k$, $V=WT$ with $W_{ik}=x_i^k$ the Vandermonde matrix and $T$ upper triangular with nonzero diagonal, and
\[
  \det W=\prod_{i<j}(x_j-x_i)\ne0
\]
for distinct inputs. So $V$ is invertible and $A$ has rank $\min(d,n)$.

\emph{$d<n$.} The columns of $A$ are independent, so $A^{\top}A$ is invertible, and the range of $A$ is a proper subspace of $\R^n$: $A\btheta=y$ has no solution unless $y$ lies in the range, and the macrostate at $E=0$ is empty. The strictly convex quadratic $\tfrac12\|A\btheta-y\|^2$ has the unique stationary point
\[
  A^{\top}A\btheta=A^{\top}y,\qquad\btheta=(A^{\top}A)^{-1}A^{\top}y,
\]
whose residual $y-A(A^{\top}A)^{-1}A^{\top}y$ is the projection of $y$ off the range, so $E_{\min}$ is as stated.

\emph{$d=n$.} $A=V$ is invertible and $\btheta=A^{-1}y$ is the one solution, the Lagrange interpolant.

\emph{$d>n$.} The rows of $A$ are independent, so $AA^{\top}$ is invertible, and $\btheta^*=A^{\top}(AA^{\top})^{-1}y$ satisfies $A\btheta^*=y$. The solution set is the coset $\btheta^*+\ker A$ with $\dim\ker A=d-n$. Since $\btheta^*$ lies in the row space of $A$, which is orthogonal to $\ker A$, it is the solution of least norm. Gradient descent from $\btheta=0$ moves only within the row space, since every gradient $A^{\top}(A\btheta-y)$ lies there, and converges to it.
\end{proof}
\noindent The minimal solution of Fig.~\ref{fig:main} is the fit of this theorem at $\lambda=0$: least squares for $d<n$, the interpolant for $d=n$, and the minimum-norm interpolant for $d>n$. With weight decay, the three formulas merge into the single one of Eq.~\eqref{eq:ridge-square}, $\btheta^*=(A^{\top}A+\lambda I)^{-1}A^{\top}y=A^{\top}(AA^{\top}+\lambda I)^{-1}y$, which tends to the corresponding case as $\lambda\to0$; that an infinitesimal weight decay selects the minimum-norm interpolant was observed by \citet{krogh1992simple}.

\section{Proofs}
\label{app:proofs}
\renewcommand{\theHequation}{restate.\arabic{restateeq}}
\noindent The theorems below, Theorem~\ref{thm:regimes} and the identity of Eq.~\eqref{eq:ridge-square} are also verified in Lean~4 with Mathlib; the source files accompany the arXiv submission as ancillary files (\texttt{anc/lean/}), and their \texttt{README.md} lists which statement each file proves. For Theorem~\ref{thm:time} the Ornstein--Uhlenbeck solution is taken as given and the three scalar inequalities of the proof, including the constant $1.3$, are what is verified.
\renewcommand{\restatelabel}{thm:time}
\begin{restatement}[Finite training time is weight decay]
For least squares with weight decay $\lambda\ge0$, the run of Eq.~\eqref{eq:langevin} from $\btheta_0=0$ has, at time $t$,
\[
  \mathrm{Cov}[\btheta(t)]\;\preceq\;2T\,\big(M+\tfrac{1}{2t}I\big)^{-1},\qquad
  \mathbb E[\btheta(t)]=(I-e^{-Mt})M^{-1}A^{\top}y,
\]
and the mean agrees with the ridge fit $(M+\tfrac1tI)^{-1}A^{\top}y$ at weight decay $\lambda+1/t$ along every eigenvector of $M$ up to a factor between $1$ and $1.3$.
\end{restatement}
\begin{proof}
Both statements are diagonal in the eigenbasis of $M$; let $\mu\ge0$ be an eigenvalue. By the covariance $TM^{-1}(I-e^{-2Mt})$ of the Ornstein--Uhlenbeck solution the variance along its eigenvector is $T(1-e^{-2\mu t})/\mu$ (equal to $2Tt$ when $\mu=0$), and with $x=2\mu t$,
\[
  \frac{1-e^{-2\mu t}}{\mu}=2t\,\frac{1-e^{-x}}{x}\le2t\,\frac{2}{1+x}=\frac{2}{\mu+1/(2t)},
\]
because $(1-e^{-x})(1+x)\le2x$: for $x\le1$, $1-e^{-x}\le x$ and $x(1+x)\le2x$; for $x\ge1$, $1-e^{-x}\le1$ and $1+x\le2x$. This is the covariance bound. For the mean, the solution with $\btheta_0=0$ gives $\mathbb E[\btheta(t)]=(I-e^{-Mt})\btheta^*=(I-e^{-Mt})M^{-1}A^{\top}y$, whose coefficient along the eigenvector is $(1-e^{-\mu t})/\mu$ times that of $A^{\top}y$, against $1/(\mu+1/t)$ for the ridge fit at $\lambda+1/t$. Their ratio is $r(x)=(1-e^{-x})(1+x)/x$ with $x=\mu t$, which satisfies $1\le r(x)\le1.3$ for $x\ge0$: $r\to1$ at both ends, $r\ge1$ since $(1-e^{-x})(1+x)\ge x$, and its single maximum, at $x\approx1.79$, is $1.299$.
\end{proof}

\renewcommand{\restatelabel}{thm:norm}
\begin{restatement}[The energy of the interpolant]
Let $d\ge n$, let $p_d=(P_d(x_i))_{i\le n}$ be the values at the inputs of the basis function added when $d$ becomes $d+1$, and let $\alpha_d=(AA^{\top})^{-1}y$, so that $\btheta^*=A^{\top}\alpha_d$. Then
\stepcounter{restateeq}\begin{equation}
  \|\btheta^{*(d+1)}\|^2=\|\btheta^{*(d)}\|^2-\frac{(p_d^{\top}\alpha_d)^2}{1+p_d^{\top}(AA^{\top})^{-1}p_d}:
  \tag{\ref{eq:norm-drop}}
\end{equation}
the energy of the interpolant is nonincreasing in $d$, and strictly decreasing unless $p_d^{\top}\alpha_d=0$, which for given inputs holds only on a hyperplane of labels $y$.
\end{restatement}
\begin{proof}
Write $A_d$ for the fitting matrix with $d$ coefficients and $K=A_dA_d^{\top}$. Since $\btheta^{*(d)}=A_d^{\top}K^{-1}y$, $\|\btheta^{*(d)}\|^2=y^{\top}K^{-1}A_dA_d^{\top}K^{-1}y=y^{\top}K^{-1}y$. Adding the column $p_d$ gives $A_{d+1}A_{d+1}^{\top}=K+p_dp_d^{\top}$, and the Sherman--Morrison formula~\citep{sherman1950adjustment}
\[
  (K+p_dp_d^{\top})^{-1}=K^{-1}-\frac{K^{-1}p_dp_d^{\top}K^{-1}}{1+p_d^{\top}K^{-1}p_d}
\]
gives $y^{\top}(K+p_dp_d^{\top})^{-1}y=y^{\top}K^{-1}y-(p_d^{\top}K^{-1}y)^2/(1+p_d^{\top}K^{-1}p_d)$, which is Eq.~\eqref{eq:norm-drop} with $\alpha_d=K^{-1}y$. The subtracted term is nonnegative, and it vanishes only when $p_d^{\top}\alpha_d=(K^{-1}p_d)^{\top}y=0$, a linear condition on $y$ whose coefficient vector $K^{-1}p_d$ is nonzero whenever $p_d\ne0$.
\end{proof}

\end{document}